\documentclass[conference,letterpaper]{IEEEtran}

\usepackage[utf8]{inputenc} 
\usepackage[T1]{fontenc}
\usepackage{url}
\usepackage{ifthen}
\usepackage{cite}
\usepackage{amsfonts}
\usepackage{xcolor}
\usepackage[cmex10]{amsmath} 
\usepackage{amsthm}

\theoremstyle{Conjecture}
\newtheorem{conjecture}{Conjecture}
\theoremstyle{Lemma}
\newtheorem{lemma}{Lemma}
\theoremstyle{Remark}
\newtheorem{remark}{Remark}

\begin{document}
\title{Log-Concave Coupling for\\Sampling Neural Net Posteriors} 


\author{%
  \IEEEauthorblockN{Curtis~McDonald}
  \IEEEauthorblockA{Department of Statistics and Data Science\\
                    Yale University\\
                    New Haven, CT, USA\\
                    Email: curtis.mcdonald@yale.edu}
  \and
   \IEEEauthorblockN{Andrew R.~Barron}
  \IEEEauthorblockA{Department of Statistics and Data Science\\
                    Yale University\\
                    New Haven, CT, USA\\
                    Email: andrew.barron@yale.edu}
}


\maketitle


\begin{abstract}
  In this work, we present a sampling algorithm for single hidden layer neural networks. This algorithm is built upon a recursive series of Bayesian posteriors using a method we call Greedy Bayes. Sampling of the Bayesian posterior for neuron weight vectors $w$ of dimension $d$ is challenging because of its       
   multimodality. Our algorithm to tackle this problem is based on a coupling of the posterior density for $w$ with an auxiliary random variable $\xi$.
   
   The resulting reverse conditional $w|\xi$ of neuron weights given auxiliary random variable is shown to be log concave. In the construction of the posterior distributions we provide some freedom in the choice of the prior. In particular, for Gaussian priors on $w$ with suitably small variance, the resulting marginal density of the auxiliary variable $\xi$ is proven to be strictly log concave for all dimensions $d$. For a uniform prior on the unit $\ell_1$ ball, evidence is given that the density of $\xi$ is again strictly log concave for sufficiently large $d$. 
   
   The score of the marginal density of the auxiliary random variable $\xi$ is determined by an expectation over $w|\xi$ and thus can be computed by various rapidly mixing Markov Chain Monte Carlo methods. Moreover, the computation of the score of $\xi$ permits methods of sampling $\xi$ by a stochastic diffusion (Langevin dynamics) with drift function built from this score. With such dynamics, information-theoretic methods pioneered by Bakry and Emery show that accurate sampling of $\xi$ is obtained rapidly when its density is indeed strictly log-concave. After which, one more draw from $w|\xi$, produces neuron weights $w$ whose marginal distribution is from the desired posterior.\footnote{This research was presented at the International Symposium on Information Theory (ISIT). Athens, Greece, July 11, 2024. The material was also presented in the 2024 Shannon Lecture.}
\end{abstract}

\section{Introduction}
Bayesian methods for parameterized models have long been prized by statisticians for various reasons. Maximum Likelihood Estimation (MLE) provides only a single point estimate among all the models of a given class, while a Bayesian posterior provides a full distribution over all possible model parameters. As such, a posterior mean is a mixing of many different models in a class, compared to a single point estimate, and can have much richer estimation properties than any single model. Furthermore, MLE requires optimization of what can be a potentially multimodal surface, whereas Bayesian posterior sampling can potentially overcome that challenge. Bayesian methods have smooth transition from the prior distribution to the posterior compared to single point estimates, are more robust to model inaccuracy by balancing a mixture of models in their posterior means, and are more amenable to predictive risk bounds via information theoretic analysis.

The computational barrier to implementing an effective Bayesian model is computing the resulting posterior means. Such means are usually computed via the empirical average of a Markov Chain Monte Carlo (MCMC) sampling algorithm. In order to be sampled efficiently, one requires a guarantee of rapid mixing of an MCMC algorithm for the posterior distribution in a polynomial number of iterations dependent on the dimension of the parameters $d$ and number of observed data points $n$.

For continuous parameter values, algorithms where rapid mixing of MCMC methods is established are focused on probability distributions with a log concave probability density function either unrestricted over $\mathbb{R}^{d}$ or restricted over a convex set. In such a situation, common devices for establishing rapid mixing such as log-Sobolev inequalities, conductance bounds, and spectral conditions follow nicely. This often results in an exponential decay in the relative entropy $D_{t}\leq D_{0}e^{-c t}$ at some rate $c$ along the Markov process.

However, in the realm of modern machine learning with sufficiently complex models and inherent non-linearity, the resulting posterior for a Bayesian model will often exhibit multi-modality and a non-concave landscape for the log likelihood. Thus, it is not guaranteed to be efficiently sampled by existing MCMC methods.

Therefore, one is left with the difficulty of how to compute the posterior means necessary to follow a Bayesian approach for modern machine learning algorithms. In this paper, we study a class of posterior distributions for single hidden layer neural networks with $K$ neurons and neuron weights $w_{k} \in \mathbb{R}^{d}$, for $k \in \{1, \cdots, K\}$. The posterior distributions $p(w)$ on neuron weights we study are not themselves log concave. However, by coupling with a specifically chosen auxiliary random variable $\xi$ with predefined forward conditional density $p(\xi|w)$, we can construct a joint density ${p(w, \xi) = p(w) p(\xi|w)}$. The joint density can also be expressed via the resulting marginal density $p(\xi)$ on $\xi$ and reverse conditional density $p(w|\xi)$, with $p(w, \xi) = p(w|\xi)p(\xi)$. 

The key insight of this work is that with properly chosen auxiliary random variable $\xi$, the reverse conditional density $p(w|\xi)$ can be shown to be log concave. The authors explore the question of the log concavity of the marginal density $p(\xi)$ and relate this matter to a comparison of conditional variances of linear combinations of $w$ given $\xi$ that arise from the posterior to those under the broader prior. When $p(\xi)$ is strictly log concave an efficient draw of $\xi \sim p(\xi)$ can be made, and a resulting draw $w \sim p(w|\xi)$ can be made thereafter resulting in a draw from the original posterior density for $w$ using only log concave sampling methods.

These densities can be used to construct a recursive series of posterior means based on the residuals of the previous fit in a method the authors call the Greedy Bayes estimator. With the expectations defined in this method expressed via log concave densities, they can be sampled in low order polynomial number of iterations. Using the information-theoretic techniques of (\!\!\cite{Barron1998ValenciaConference},\cite{YangBarron1998}) these estimators have beneficial predictive risk bounds, which will be detailed in future work.

\section{Model Parameters and Auxiliary Random Variable Distribution}\label{model_specifics}
Let $d$ be the dimension of the input covariates. We have $n$ pairs of input data $x_{i} \in \mathbb{R}^{d}$ with response values $y_{i}$ for $i \in \{1, \cdots, n\}$. The $y_{i}$ given $x_{i}$ is defined by some function $f(x_{i})$ which is not known to us and we wish to estimate from our observations.

Say we have a neuron activation function $\psi$ which is continuous in its first derivative and has bounded second derivative $|\psi''(u)| \leq c$ for all $u \in \mathbb{R}$. This includes, for example, the tanh function and the squared ReLU.

Define a neuron weight vector $w \in \mathbb{R}^{d}$. Suppose at present, we have some existing  fit $\hat{f}_{i}$ for each data observation $y_{i}$. These fits could come from some common fitting function $\hat{f}$ applied to each point $x_{i}$, $\hat{f}_{i} = \hat{f}(x_{i})$, but are not required to. For some mixture weight $\beta \in (0,1)$, we want to create an updated fit for each data point by down-weighting the previous fit and incorporating in some small amount of a new neuron,
\begin{align}
\hat{f}^{\text{new}}_{i}&= (1-\beta)\hat{f}_{i}+\beta \,\psi(x_{i} \cdot w),
\end{align}
where $x_{i} \cdot w$ denotes the inner product. When this parameter $w$ is found by a least squares optimization, then this is the relaxed greedy neural net approach of Barron \cite{barron1993universal, 10.1214/009053607000000631, huang2008risk, klusowski2016risk} and Jones \cite{MR1150368}, and has connections to projection pursuit \cite{MR1150368} and boosting algorithms \cite{10.1214/aos/1013203451}. As an important example, we could consider the previous fit $\hat{f}_{i}$ as a $K-1$ wide neural net $\hat{f}_{i}= \sum_{k=1}^{K-1}c_{k}\psi(x_{i}\cdot w_{k})$ and we wish to add in one new neuron.

Define ${r_{i} = y_{i} - (1-\beta)\hat{f}_{i}}$ as the residuals of our previous fit, and given a prior density $p_{0}(w)$ we define the Greedy Bayes posterior for some scaling parameter $\alpha \in (0,1)$ as,
\begin{align}
p(w) \propto \text{exp} \left(\alpha\sum_{i=1}^{n} r_{i}\psi(x_{i}\cdot w)\right)p_{0}(w). \label{target_density}
\end{align}
This density prioritizes weights $w$ which have high inner product under the activation function with the residuals. Constant order $\alpha$ has more favorable risk properties, whereas smaller order $\alpha$ gives easier proof of efficient sampling methods. In separate study, reasonable risk control occurs as long as $\alpha$ is at least of order $n^{-\frac{1}{2}}$.

We consider two different priors for $p_{0}(w)$ and corresponding assumptions on the data matrix $X$:
\begin{enumerate}
\item Assume $|x_{i,j}| \leq 1$ for all data matrix entries, and use a uniform prior for weights $w$ over the set of $\ell_{1}$ norm less than 1, $C=\{w:\|w\|_{1} \leq 1\}$.
\item With control of the largest eigenvalue of $X^TX$, we use a normal prior $p_{0} = N(0,\sigma_{0}^{2} I)$ with variance $\sigma_{0}^{2}$.  
\end{enumerate}
We will denote the prior as $p_{0}(w)$ and specialize to each specific case when necessary.

The density (\ref{target_density}) is by itself not generally log concave. Indeed, checking the Hessian of $\log p(w)$ it is a linear combination of rank 1 matrices with positive and negative multiples plus a contribution from the prior, 
\begin{align}
\nabla^{2}\log p(w)&= \alpha \sum_{i=1}^{n}r_{i}\psi''(x_{i} \cdot w)x_{i}x_{i}^{T}+\nabla^{2}\log p_{0}(w),
\end{align}
where the $x_{i}$ are interpreted as column vectors with outer product $x_{i}x_{i}^{T}$. The prior contribution is either $\nabla^{2}\log p_{0}(w)=0$ in the uniform case or $\nabla^{2} \log p_{0}(w) = -\frac{1}{\sigma_{0}^{2}} I$ in the Gaussian case. Neither case is guaranteed to overpower the contributions from the rank one matrices, so the overall expression could have both negative or positive eigenvalues at different $w$ inputs and is not a negative definite matrix. Therefore, we introduce an auxiliary random variable as a tool to overcome this non-concavity. 

We define the $n$ dimensional random variable $\xi$ by,
\begin{align}
\xi_{i}= (\alpha c |r_{i}|)^{\frac{1}{2}}x_{i}w+Z_{i},\quad Z_{i} \sim N(0,1),
\end{align}
with $Z_{i}$ an independent normal random variable . With $w \sim p(w)$ this coupling defines the forward conditional density for $p(\xi|w)$. Combining these two densities gives the joint density $p(w, \xi) = p(w)p(\xi|w)$ as proportional to,
\begin{small}
\begin{align}
p_{0}(w)\text{exp}\!\left(\alpha\sum_{i=1}^{n}r_{i}\psi(x_{i}\!\cdot \!w) -\frac{1}{2}\!\sum_{i=1}^{n}(\xi_{i}\!-\!(\alpha c|r_{i}|)^{\frac{1}{2}}x_{i}\!\cdot\!w)^{2}\right).
\end{align}
\end{small}
By expanding the quadratic form this can be expressed as,
\begin{align*}
&p_{0}(w)\text{exp} \big(\sum_{i=1}^{n}\alpha r_{i}\psi(x_{i}\!\cdot\!w)-\frac{\alpha c|r_{i}|}{2}(x_{i}\!\cdot\! w)^{2}\big)\\
&\text{exp}\big( \sum_{i=1}^{n}[-\frac{1}{2}\xi_{i}^{2}+(\alpha c |r_{i}|)^{\frac{1}{2}}\xi_{i}x_{i}\!\cdot\!w] \big).
\end{align*}
For notational convenience define,
\begin{align}
g(w)&= \sum_{i=1}^{n}\alpha r_{i}\psi(x_{i}w)-\frac{\alpha c |r_{i}|}{2}(x_{i}w)^{2}.\label{g_def}
\end{align}
The joint density for $p(w, \xi)$ can be written in two ways, the forward expression $p(w)p(\xi|w)$ and the reverse expression $p(\xi)p(w|\xi)$ using the induced marginal density $p(\xi)$ for $\xi$ and the reverse conditional density $p(w|\xi)$for $w|\xi$.
The resulting conditional density $p(w|\xi)$ is,
\begin{align}
p(w|\xi)&\propto \text{exp} \left(g(w)+\sum_{i=1}^{n}(\alpha c |r_{i}|)^{\frac{1}{2}}\xi_{i} x_{i}\!\cdot\! w \right)p_{0}(w).
\end{align}
The resulting marginal on $p(\xi)$ is, 
\begin{align}
&p(\xi)\!\propto \!e^{-\frac{1}{2}\sum_{i=1}^{n} \xi_{i}^{2}}\!\int e^{g(w)+\sum_{i=1}^{n}(\alpha c |r_{i}|)^{\frac{1}{2}}\xi_{i} x_{i}\!\cdot\! w}p_{0}(w)dw.
\end{align}

\section{The Log Concavity of Densities $p(w|\xi)$ and $p(\xi)$}\label{log_concave}

\subsection{Reverse Conditional Density $p(w|\xi)$}
The exponent of  $p(w|\xi)$ is composed of three parts (we ignore a constant here as we only study the density up to proportionality),
\begin{small}
\begin{align}
\log p(w|\xi)&= g(w)+\sum_{i=1}^{n}(\alpha c |r_{i}|)^{\frac{1}{2}}\xi_{i} x_{i}\!\cdot\!w+\log p_{0}(w)+K_{\xi}.
\end{align}
\end{small}
The Hessian is then,
\begin{align}
&\nabla^{2}\log p(w|\xi)= \nabla^{2}g(w)+\nabla^{2}\log p_{0}(w)\\
&=\alpha\sum_{i=1}^{n}|r_{i}|(\text{sign}(r_{i})\psi''(x_{i}\!\cdot\!w)- c)x_{i}x_{i}^{T}+\nabla^{2}\log p_{0}(w).
\end{align}
By assumption $|\psi''(u)|\leq c$ for any input $u$, so the above is a sum of negative multiples of rank one matrices plus a negative definite prior contribution, so it is a negative definite expression. Thus the log density of $p(w|\xi)$ is a concave function for any conditioning value $\xi$.

Note for the Gaussian prior case (where $w$ is unbounded) the presence of the Hessian from the prior makes $p(w|\xi)$ strictly log concave, as will be needed for rapid mixing of Langevin dynamics in this case, whereas for the uniform prior on the compact $C$, the log concavity need not be strict as sampling methods mix rapidly for log concave densities on compact sets.

\subsection{Marginal Density $p(\xi)$}

The log density of the marginal $p(\xi)$ has a quadratic term in $\xi$ and a term which represents the cumulant generating function of $w$ under the density $\tilde{p}(w) \propto e^{g(w)}p_{0}(w)$. That is, $\log p(\xi)$ is given by,
\begin{small}
\begin{align}
\log p(\xi)&= -\frac{1}{2}\sum_{i=1}^{n}\xi_{i}^{2} +\log \int e^{\sum_{i=1}^{n}(\alpha c |r_{i}|)^{\frac{1}{2}}\xi_{i} x_{i}\cdot w}\tilde{p}(w)dw+K.
\end{align}
\end{small}

Denote $|R|$ as the diagonal matrix of absolute values of the residuals.  The score is then a linear term in $\xi$ and the conditional expectation under the reverse conditional distribution,
\begin{align}
\nabla \log p(\xi)&= - \xi+E[(\alpha c |R|)^{\frac{1}{2}}Xw|\xi].
\end{align}

Important to the implementation of MCMC samplers of $p(\xi)$ is that we are able to compute it's score. Fortunately, the score function has the desired property that it is defined by an expected value over the previously defined log concave distribution for $w|\xi$. The computation of this expectation is facilitated by MCMC samples from the log concave density $p(w|\xi)$.

The Hessian of $\log p(\xi)$ is the negative identity matrix plus the conditional covariance matrix of $(\alpha c |R|)^{\frac{1}{2}}Xw$,
\begin{align}
\nabla^{2} \log p(\xi)&= -I+\text{Cov}[(\alpha c |R|)^{\frac{1}{2}}Xw|\xi].\label{original_cov_expression}
\end{align}

In order for this to be a negative definite matrix, we need the largest eigenvalue of $\text{Cov}[(\alpha c |R|)^{\frac{1}{2}}Xw|\xi]$ to be less than 1. This is equivalent to the statement that for any unit vector $a$, the scalar random variable $z = a^{T}(\alpha c |R|)Xw$ has variance less than 1, $\text{Var}(z|\xi) \leq 1~\forall~\xi \in \mathbb{R}^{n}$. We then study the log concavity of this density under the two different assumptions on the data matrix and prior.
\subsubsection{Gaussian Prior and Data Matrix Eigenvalues}~\\
Let $X^{T}X$ have largest eigenvalue $\lambda_{\max}$. In this section, we will prove the following:
\begin{enumerate}
\item[a)] Using a Gaussian prior with small variance $\sigma_{0}^{2} \leq \frac{1}{\alpha c \|r\|_{\infty} \lambda_{\max}}$ results in $p(\xi)$ being log concave.
\item[b)] There exist larger variances $\sigma_{0}^{2}>\frac{1}{\alpha c \|r\|_{\infty} \lambda_{\max}}$ that result in $p(\xi)$ being log concave.
\end{enumerate}
\begin{lemma}\label{guassian_contraction}
The conditional covariance matrix of the density $p(w|\xi)$ under the Gaussian prior is dominated by the covariance matrix of the prior,
\begin{align}
\text{Cov}[w|\xi] \preceq \sigma_{0}^{2} I.
\end{align}
Equivalently, for any direction $v$ the variance of $z = v\cdot w$ is less than $\sigma_{0}^{2}\|v\|^{2}$,
\begin{align}
\text{Var}(v \cdot w|\xi) \leq \sigma_{0}^{2}\|v\|^{2}.
\end{align}
\end{lemma}
\begin{proof}
The log density for $p(w|\xi)$ is the Gaussian prior log density plus a linear term and the concave function $g(w)$. From the results of Caffarelli \cite{caffarelli2000monotonicity} and Chewi and Pooladian \cite{ChePoo23Caffarelli}, we have that over the whole of $\mathbb{R}^{d}$, for two densities $p(w) \propto e^{-V(w)}$ and $q(w) =e^{-V(w)-G(w)}$ where $V, G$ strictly convex functions, there exists as transport map from $p$ to $q$ that is a contraction. Restricting to one dimensional directions $z = v\cdot w$, the one dimension density for $z$ when $w$ is drawn from $p(w|\xi)$ is more log concave than when $w$ is drawn from the prior. As such, the transport map for scalar random variable $z$ is a contraction. Therefore for any direction $v$ the variance of $z = v \cdot w$ is less when $w$ is drawn from $p(w|\xi)$ than when $w$ is drawn from the prior.
\end{proof}
\begin{lemma}
Using a Gaussian prior with variance $\sigma_{0}^{2} \leq \frac{1}{\alpha c \|r\|_{\infty} \lambda_{\max}}$, the density $p(\xi)$ is log concave.
\end{lemma}
\begin{proof}
For any unit vector $a \in \mathbb{R}^{d}$, by Lemma \ref{guassian_contraction} we have,
\begin{align*}
&a^{T}\text{Cov}[(\alpha c |R|)^{\frac{1}{2}}Xw|\xi]a\\
\leq &\sigma_{0}^{2}a^{T}(\alpha c |R|)^{\frac{1}{2}}XX^{T}(\alpha c |R|)^{\frac{1}{2}}a\\
\leq &\sigma_{0}^{2}(\alpha c \|r\|_{\infty}\lambda_{\max})\|a\|^{2}\\
\leq& \|a\|^{2}.
\end{align*}
This results in expression (\ref{original_cov_expression}) being negative definite and thus $p(\xi)$ is log concave.
\end{proof}
While Lemma (\ref{guassian_contraction}) is true, bounding $\text{Cov}[w|\xi]$ by the prior variance alone is a simple but loose bound. By using more involved analysis, we can show that there are higher variances $\sigma_{0}^{2} > \frac{1}{\alpha \|r\|_{\infty} c \lambda_{\max}}$ that result in the log concavity of $p(\xi)$ as well.
\begin{lemma}
Let $X$ have singular value decomposition $X = U \Lambda V^{T}$ with $\lambda_{\max} = \max_{i\in\{1,d\}}\lambda_{i}^{2}$. Denote the diagonal matrix of residuals as $R$, absolute value residuals as $|R|$, and define the diagonal matrix $S(w)$ with entries,
\begin{align}
[S(w)]_{i,i} = \psi''(x_{i}w).
\end{align}
Define the matrices $A \in \mathbb{R}^{n \times d}, B, C(w) \in \mathbb{R}^{d \times d}$,
\begin{align}
A&= (\alpha c |R|)^{\frac{1}{2}}U,\quad\quad B= U^{T}(\alpha c |R|)U,\\
C(w)&= \alpha U^{T}RS(w)U.
\end{align}
Then we have upper bound on the Hessian of $p(\xi)$ as
\begin{small}
\begin{align}
\nabla^{2}\log p(\xi)& \preceq AE\left[\left. -B^{-1}+(\frac{1}{\sigma_{0}^{2}}\Lambda^{-2}-C(w)+B)^{-1}\right|\xi\right]A^{T}.\label{matrix_terms}
\end{align}
\end{small}

If $\sigma_{0}^{2}=\frac{1}{\alpha c \|r\|_{\infty}\lambda_{\max}}$ then (\ref{matrix_terms}) is a negative definite matrix and we have $ \nabla^{2}\log p(\xi) \prec 0$. The expression is continuous in $\sigma_{0}^{2}$ thus there exists values $\sigma_{0}^{2}>\frac{1}{\alpha c \|r\|_{\infty}\lambda_{\max}}$ that achieve negative definiteness as well.
\end{lemma}
\begin{proof}
Using integration by parts, an equivalent expression for the covariance of $w|\xi$ is,
\begin{align}
\text{Cov}[w|\xi]&= \sigma_{0}^{2}I+\sigma_{0}^{4}\left(E[\nabla^{2}g(w)|\xi]+\text{Cov}(\nabla g(w)|\xi)\right).\label{IBP}
\end{align} 
Since $w|\xi$ is log concave we have a Brascamp-Lieb inequality \cite{bobkov2000brunn} upper bounding this covariance term,
\begin{align}
&\text{Cov}[\nabla g(w)|\xi]\\
\preceq &E[(-\nabla^{2}g(w))(\frac{1}{\sigma_{0}^{2}}I-\nabla^{2}g(w))^{-1}(-\nabla^{2}g(w))|\xi]\\
= &E[-\nabla^{2} g(w)-\frac{1}{\sigma_{0}^{2}}I+\frac{1}{\sigma_{0}^{4}}(\frac{1}{\sigma_{0}^{2}}I-\nabla^{2}g(w))^{-1}|\xi].\label{last_line}
\end{align}
Combining (\ref{last_line}) with (\ref{IBP}) gives the upper bound,
\begin{align}
\text{Cov}[w|\xi]& \preceq E[(\frac{1}{\sigma_{0}^{2}}I-\nabla^{2} g(w))^{-1}|\xi].
\end{align}
This upper bound can then be input to equation (\ref{original_cov_expression}) and via matrix algebra expressed as equation (\ref{matrix_terms}).

The only part of the expectation in (\ref{matrix_terms}) changing in $w$ is the matrix $C(w)$. Note that since $\psi$ has bounded second derivative,
\begin{align}
C(w)& \preceq \alpha c U^{T}|R|U \preceq \alpha c \|r\|_{\infty}I
\end{align}
With prior variance $\sigma_{0}^{2} = \frac{1}{\alpha \|r\|_{\infty} c \lambda_{\max}}$ the term $\frac{1}{\sigma_{0}^{2}}\Lambda^{-2}-C(w)$ represents a positive semi definite matrix for any choice of $w$. As such, the inverse of a matrix plus a positive semi definite matrix is dominated by the inverse of the standalone matrix,
\begin{align*}
\frac{1}{\sigma_{0}^{2}}\Lambda^{-2}-C(w) \succeq 0 \implies (\frac{1}{\sigma_{0}^{2}}\Lambda^{-2}-C(w)+B)^{-1}&\preceq B^{-1},
\end{align*}
so the term in the expectation of (\ref{matrix_terms}) is negative semi definite for all input $w$ values.

The term in the expectation is zero only at those $w$ values where $\frac{1}{\sigma_{0}^{2}}\Lambda^{-2}= C(w)$. At all other $w$ values, the term is strictly negative definite. As this set is not a probability one event, the expectation must be some finite amount below the 0 matrix. Thus, we can increase the prior variance  $\sigma_{0}^{2}$ to some amount above the value $\frac{1}{\alpha c \|r\|_{\infty} \lambda_{\max}}$ and still maintain negative definiteness for these values.
\end{proof}
\subsubsection{Bounded Data Entries and Uniform Prior over $\ell_{1}$ Ball}

In the case of the uniform prior over the $\ell_{1}$ ball, we would like to give a contraction result similar to Lemma \ref{guassian_contraction}. However, for a log concave distribution restricted to a convex set, the one dimensional marginals are more complicated as the geometry of the convex set can impact the Hessian of the scalar distributions. Therefore, the authors leave the equivalent result for the uniform prior as a conjecture to be proven in future work.
\begin{conjecture}\label{conj_1}
The covariance matrix of the density $p(w|\xi)$ under the uniform prior over the $\ell_{1}$ ball is dominated by the covariance matrix of the prior,
\begin{align}
\text{Cov}[w|\xi] \preceq \text{Cov}_{\text{Uni}(C)}(w).\label{cov_dom_eqn}
\end{align}
Equivalently, for any direction $v$ the variance of $z = v\cdot w$ is less under $w$ drawn from $p(w|\xi)$ than $w$ drawn uniformly
\begin{align}
\text{Var}(v \cdot w|\xi) \leq \text{Var}_{\text{Uni}(C)}(v \cdot w).
\end{align}
\end{conjecture}
\begin{lemma}
If conjecture \ref{conj_1} holds, if the dimension satisfies $d > \alpha\,c\,n\,\|r\|_{\infty}$ and if  $|x_{i,j}| \leq 1$ for all data entries then $p(\xi)$ is strictly log concave.
\end{lemma}
\begin{proof}
Consider the covariance of $w$ drawn uniformly from the $\ell_{1}$ ball.  When drawn uniformly, $\text{Var}(w_{j})=\frac{d}{(d+1)^{2}(d+2)} \leq \frac{1}{d^{2}}$ and $\text{Cov}_{\text{Uni}(C)}[w_{j_{1}}, w_{j_{2}}]=0$ for $j_{1} \neq j_{2}$. This follows from properties of the Dirichlet distribution. For any unit vector $a$,
\begin{align*}
&\alpha c a^{T}|R|^{\frac{1}{2}}X\text{Cov}_{\text{Uni}(C)}[w]X^{T} |R|^{\frac{1}{2}}a\leq \frac{\alpha c}{d^{2}}a^{T}|R|^{\frac{1}{2}}XX^{T}|R|^{\frac{1}{2}}a\\
&= \frac{\alpha c}{d^{2}} \sum_{i,j=1}^{n}a_{i}a_{j}|r_{i}|^{\frac{1}{2}}|r_{j}|^{\frac{1}{2}}x_{i}\cdot x_{j}.
\end{align*}
Note that $x_{i} \cdot x_{j} \leq d~\forall i,j$ due to bounded data assumption.
\begin{align*}
&\frac{\alpha c}{d^{2}} \sum_{i,j=1}^{n}a_{i}a_{j}|r_{i}|^{\frac{1}{2}}|r_{j}|^{\frac{1}{2}}x_{i}\cdot x_{j}\leq \frac{\alpha c}{d} \sum_{i,j=1}^{n}a_{i}a_{j}|r_{i}|^{\frac{1}{2}}|r_{j}|^{\frac{1}{2}}\\
&= \frac{\alpha c}{d}(\sum_{i=1}^{n}|r_{i}|^{\frac{1}{2}}a_{i})^{2}\leq \frac{\alpha c}{d}\|r\|_{1}\leq \frac{\alpha c  n \|r\|_{\infty}}{d} < 1
\end{align*}
Equation (\ref{original_cov_expression}) then shows $p(\xi)$ is strictly log concave.
\end{proof}

\begin{remark}
After the advance publication of this work \cite{mcdonald2024logconcave}, we developed an alternative proof for the log-concavity of the marginal density which we presented at the 2024 International Symposium on Information Theory (ISIT) \cite{barron2024shannon}. The proof uses a H{\"o}lder inequality requiring a ratio of $20 \frac{(\alpha c n \|r\|_{\infty})^{2}}{d}<1$ to achieve log concavity. Using an $\alpha$ of $\frac{1}{\sqrt{n}}$, this requires dimension $20 (c \|r\|_{\infty})^{2}n < d$, while the conjectured result only requires $c \|r\|_{\infty} \sqrt{n} < d$. This is a weaker result, but the authors where able to achieve a proof while the covariance domination  condition (\ref{cov_dom_eqn}) presented here remains a conjecture.

The H{\"o}lder argument goes as follows. Define $g_{\xi}(w)$ as the log density of $p(w|\xi)$, and define $\tilde{g}_{\xi}(w)$ as the log density shifted by it's mean under the prior. Define $\Gamma_{\xi}(\alpha)$ as the cumulant generating function of $\tilde{g}_{\xi}(w)$ under the prior at a given $\alpha$ level. Then for any direction $v$, by a H{\"o}lder inequality with parameter $\ell$ we have upper bound:
\begin{small}
\begin{align}
\text{Var}[v\cdot w |\xi]& \leq (E_{p_{0}}[(v\!\cdot\!w)^{2\ell}])^{\frac{1}{\ell}}\text{exp}\left\{\frac{\ell}{\ell-1}\Gamma_{\xi}(\frac{\ell}{\ell-1}\alpha) - \Gamma_{\xi}(\alpha)\right\}
\end{align}
\end{small}
The first term depends on the moments of the uniform prior over the $\ell_{1}$ ball, which are well understood, and the second term depends on the growth of the cumulant generating function in a high probability region of $\xi$ values. Studying both terms separately and optimizing over the choice of $\ell$ yields the stated result. This proof method will be presented completely in future work.

\end{remark}

\subsection{Connections with Reverse Diffusion}
The authors initially came up with this coupling while studying score based diffusion \cite{song2020score, tzen2019theoretical} as a sampling method. Consider $\xi_{0} = \tilde{X}w$ with $w$ drawn from $p(w)$ and then following the SDE $ d\xi_{\tau}= -\xi_{\tau} d\tau+\sqrt{2}dB_{\tau}$ converging to standard normal for large $\tau$. This also would induce forward conditional $p(\xi_{\tau}|\xi_{0} = \tilde{X}w) \sim N(e^{-\tau}\tilde{X}w, (1-e^{-2\tau})I)$. The idea of score based diffusion is if one can compute the scores of the induced marginals  $\nabla \log p(\xi_{\tau})$ one can implement a reverse SDE that takes samples from a standard normal to samples from $\tilde{X} w, w \sim p(w)$. As discussed above, the scores of the marginals $p(\xi_{\tau})$ can be computed via expectations over the reverse condition distributions $p(\xi_{0}|\xi_{\tau})$. The authors noticed for values of $\frac{e^{-2\tau}}{1-e^{-2\tau}} \geq \alpha c$ the reverse conditional which defines this expectation is log concave, and thus these scores as expectations could be computed via MCMC averages as discussed above. However, we have to be able to sample the marginal $p(\xi_{\tau})$ at time $\tau$ to initialize the reverse process, and this seems only feasible if this density is itself log concave or near a log concave density at the given $\tau$ value. Therefore, if one can show $p(\xi_{\tau})$ and $p(\xi_{0}|\xi_{\tau})$ are both log concave, we can remove the apparatus of the reverse SDE altogether and simply get draws from $p(\xi_{\tau})$ and $p(\xi_{0}|\xi_{\tau})$ to sample our original distribution. This intuition leads to the coupling we define.

\section{MCMC Sampling for Log Concave Target Distributions}\label{MCMC}
In the uniform prior case, the sampling problem for $p(w|\xi)$ represents a log concave density over a constrained convex set $C$. The first polynomial time bounds for log concave sampling over convex sets come from \cite{applegate1991sampling} of order $\tilde{O}(d^{10})$ (note $\tilde{O}$ ignores $\log n$ factors). Over the years, the polynomial time bound for Hit-and-Run and Ball Walk algorithms was reduced  to yield mixing time bounds of order $\tilde{O}(d^{4})$  in \cite{lovasz2007geometry}.

Modern methods of \cite{srinivasan2023fast}, \cite{kook2023efficiently}, \cite{kook2022sampling} continue to push the polynomial mixing times bounds for log concave densities over convex sets. These methods take a sampling algorithm in the unconstrained case, e.g. Langevin diffusion or Hamiltonian Monte Carlo, and produce versions that can be applied over a constrained convex set. These methods vary in their dependence on different properties of the set in question, encapsulated in properties of a so called barrier function $\phi$, which the authors will not go into detail about here. However, these algorithms essentially obtain mixing time bounds of order $O(d^{3})$.

Our sampling problem for $p(\xi)$ represents a log concave density over the full $\mathbb{R}^{n}$ space, as does $p(w|\xi)$ in the Gaussian prior case. The results of Bakry, Emery, et al \cite{bakry1985diffusions, bakry2014analysis} study when a continuous time stochastic diffusion has exponential decay in it's relative entropy ${D(P_{t}\|P) \leq D(P_{0}\|P)e^{-\frac{c}{2} t}}$. This result relies on the fact that the derivative of the relative entropy is minus half the expected norm squared of the difference in scores, known as the relative Fisher information. The establishment of a log-Sobolev inequality shows the relative Fisher information is lower bound by a multiple of the relative entropy, which establishes exponential decay. For Langevin diffusion, the SDE with the score as the drift and with constant dispersion, the Bakry Emery condition, a sufficient condition for a log Sobolev inequality, reduces to a condition on the strict concavity of the log likelihood. If the density is $c$ strongly log concave, then the relative entropy decays at rate $e^{-\frac{c}{2}t}$.

\section{Greedy Bayes for Neural Networks}\label{risk}
We now construct a series of recursive posterior means defined by densities of the form (\ref{target_density}). For each index $i \in \{1, \cdots,n\}$ initialize fits $\hat{f}_{i,0}(x) = 0$ and residuals $r_{i,0} = y_{i}$. Set $\beta \in(0,1)$ as our update weight and $\alpha \in (0,1)$ as our sampling scaling.

Then, pick some order $K$ for our greedy fit. For all indexes $k \in \{1, \cdots, K\}$, recursively define a posterior for index $i$ using the previous residuals of indexes $j \in \{1,\cdots, i-1\}$
\begin{align}
p_{i,k}(w)& \propto \text{exp}\left(\alpha\sum_{j=1}^{i-1}r_{j,k-1}\psi(x_{j}\cdot w) \right)p_{0}(w).
\end{align}
Update the fit by the posterior mean of this distribution and define a new set of residuals
\begin{align}
\hat{f}_{i,k}(x)&= (1-\beta)\hat{f}_{i,k-1}(x)+\beta E_{p_{i,k}}[\psi(x\cdot w)]\\
r_{i,k}&= y_{i} - (1-\beta)\hat{f}_{i,k}(x_{i}).
\end{align}
{At level $k=1$, for any index $i$, $\hat{f}_{i,0}(x)=0$ thus $\hat{f}_{i,1}$ is the posterior mean,
\begin{align*}
\hat{f}_{i,1}(x)&=\beta E_{p_{i}, 1}[\psi(x\cdot w)].
\end{align*}
This mean can be computed and maintained by storing some $L$ number of samples from $p_{i,1}(w)$. Then for any desired $x$ value, $\hat{f}_{i,1}(x)$ at this value can be computed as the empirical mean of the stored $L$ weights.

Moving to levels $k > 1$ the previous estimates $\hat{f}_{i,k-1}(x)$ can be evaluated by empirical averages of stored previous samples from $p_{j,s}(w)$ densities for $j < i, s < k$ and so on.

Given a new data point $x$, we define the $K$ order Greedy Bayes estimator as
\begin{align}
\hat{f}_{K}(x)&= \frac{1}{n} \sum_{i=1}^{n}\hat{f}_{i, K}(x).
\end{align}
This amounts to a mixture of $nK$ conditional means, where each fit $\hat{f}_{i,K}$ is only a function of data $(x_{j}, y_{j})$ for $j \in \{1, \cdots, i\}$.
All the conditional means here are of the form (\ref{target_density}) which can be expressed via the coupling in terms of log concave densities and thus sampled efficiently via MCMC methods.
\section{Future Work}
In this work, the authors define the Greedy Bayes procedure and study conditions on the prior and scaling parameter $\alpha$ that give rise to provably efficient sampling. Ongoing work will analyze the risk properties of this procedure. Current work indicates that the Greedy Bayes procedure can be paired with certain priors yielding both efficient sampling as studied here and information-theoretic determination of the risk.

\bibliographystyle{IEEEtran}
\bibliography{ISIT_log_concave.bib}

\end{document}